\documentclass[runningheads]{llncs}
\usepackage[T1]{fontenc}
\usepackage{bbding}
\usepackage{graphicx}
\usepackage{xcolor}
\usepackage{hyperref}
\usepackage{url}
\usepackage{tabularx}

\usepackage{algorithm}
\usepackage{algorithmicx} 
\usepackage{algpseudocode}

\usepackage{amsmath,amsfonts,bm}

\def\eqref#1{equation~\ref{#1}}

\def\1{\bm{1}}

\DeclareMathAlphabet{\mathsfit}{\encodingdefault}{\sfdefault}{m}{sl}
\SetMathAlphabet{\mathsfit}{bold}{\encodingdefault}{\sfdefault}{bx}{n}

\newcommand{\Kset}{\mathcal{K}}

\newcommand{\duoE}{\mathbb{E}}

\newcommand{\Real}{\mathbb{R}}
\begin{document}
\title{Adversarial Bandit Optimization with Globally Bounded Perturbations to Convex Losses}
%
%

\author{Zhuoyu Cheng \Envelope\inst{1,3}\orcidID{0009-0002-6631-4929} \and
Kohei Hatano\inst{2,3}\orcidID{0000-0002-1536-1269} \and
Eiji Takimoto\inst{2}\orcidID{0000-0001-9542-2553}}
\authorrunning{Z. Cheng et al.}
%
\institute{Joint Graduate School of Mathematics for Innovation, Kyushu University, Japan  \and
Department of Informatics, Kyushu University, Japan \and
RIKEN AIP, Japan \\
\noindent\Envelope\email{cheng.zhuoyu.874@s.kyushu-u.ac.jp} \\
\email{\{hatano,eiji\}@inf.kyushu-u.ac.jp}}
\maketitle              
\begin{abstract}
We study adversarial bandit optimization in which the loss functions may be non-convex and non-smooth. In each round, the learner selects an action and observes only the loss incurred at that action. The loss consists of an underlying convex and $\beta$-smooth component and an adversarial perturbation that may be chosen after observing the learner's action. The perturbations are subject to a global budget controlling their cumulative magnitude over time.

This framework extends the globally budgeted, post-action perturbation model from underlying linear losses to general convex and $\beta$-smooth losses. For this broader class, we establish expected regret guarantees that explicitly characterize the effect of the perturbation budget.

To establish these guarantees, we modify a standard bandit optimization algorithm and develop an analysis that controls the additional regret caused by the perturbations.
In the absence of perturbations, our results reduce to regret guarantees for the standard bandit convex optimization setting with $\beta$-smooth losses.

\keywords{Bandit optimization, self-concordant barrier, Convex losses, Globally bounded perturbations}
\end{abstract}
\section{Introduction}\label{sec1}
Online convex optimization provides a fundamental framework for sequential
decision making under uncertainty, where a learner repeatedly chooses actions
from a convex decision set and incurs losses chosen by an adversary
\cite{shalev2012online,zinkevich2003online}. 
In many applications, however, the learner does not observe the entire loss
function or its gradient. Instead, only the loss value at the selected action is
revealed. This limited-feedback setting, known as bandit convex optimization
(BCO), captures decision-making problems in which querying the system is costly
or only realized outcomes are observable, such as online pricing, resource
allocation, and black-box optimization. 
Since the seminal works on continuum-armed bandits
and bandit convex optimization
\cite{kleinberg2004nearly,flaxman2005online},
regret minimization under convexity assumptions has been extensively studied.
Subsequent work has developed improved regret bounds under additional structural assumptions,
such as linearity, Lipschitz continuity, and smoothness.
\cite{abernethy2008competing,hazan2014bandit,saha2011improved,dekel2015bandit,bubeck2021kernel}.

A central assumption in much of the BCO literature is that each loss function is convex. This assumption is mathematically powerful: it makes randomized smoothing and gradient estimation compatible with mirror-descent-type updates in regret analysis. Nevertheless, in practical bandit optimization problems, the observed loss may deviate from an ideal convex model. Such deviations can arise from measurement errors, delayed or context-dependent effects, model misspecification, or adversarial perturbations added to an otherwise structured loss. As a result, the actual loss sequence may fail to be convex, even when it remains close to a sequence of convex and smooth functions. Existing regret guarantees for convex or smooth convex bandit optimization therefore do not directly apply to this setting, because even small nonconvex perturbations can invalidate the convexity-based arguments used in the analysis.

Motivated by this gap, we study bandit optimization with approximately convex
and smooth loss sequences. Specifically, we assume that each loss function can be
decomposed as
\[
    f_t(x) = h_t(x) + \sigma_t(x),
\]
where \(h_t\) is convex, \(\beta\)-smooth, while
\(\sigma_t\) represents an arbitrary perturbation. Rather than requiring the
perturbation to be uniformly small at every point and every round, we impose a
global budget condition:
\[
    \sup_{x_1,\ldots,x_T \in \mathcal K}
    \sum_{t=1}^T |\sigma_t(x_t)| \le C.
\]
This condition allows the observed losses \(f_t\) themselves to be nonconvex and
nonsmooth, while still requiring the cumulative deviation from an underlying
convex and smooth sequence to be controlled. Thus, the model interpolates
between standard smooth convex bandit optimization, recovered when \(C=0\), and
more general perturbed loss sequences when \(C>0\).

The main challenge in this setting is that the learner receives only scalar bandit feedback from the perturbed loss \(f_t(x_t)\), rather than from the underlying convex component \(h_t\). Consequently, the learner cannot directly construct an unbiased gradient estimator for \(h_t\), nor can it identify the perturbation \(\sigma_t\). The regret analysis must therefore disentangle the contribution of the structured convex components from the additional error caused by the perturbations. This separation is nontrivial: the benchmark is defined with respect to the cumulative perturbed losses, whereas the convexity-based analysis can only be applied to the underlying smooth convex components, with the perturbations handled as additional error terms.

In this paper, we study bandit optimization with $C$-approximately convex and
$\beta$-smooth function sequences. Although the observed losses may be
nonconvex and nonsmooth, they are assumed to remain close, in a cumulative
sense, to a sequence of smooth convex losses. Under this global perturbation
budget assumption, we show that sublinear expected regret can still be achieved.

Our contributions are summarized as follows.
\begin{itemize}
\item We introduce a bandit optimization model with
$C$-approximately convex and $\beta$-smooth function sequences. This
model extends standard bandit convex optimization by allowing nonconvex and
nonsmooth losses whose cumulative deviation from smooth convex losses is
controlled.

\item We propose a modification of the SCRiBLe-based bandit algorithm for
this perturbed loss model. The algorithm operates on a shrunk feasible set
and uses a scaled Dikin ellipsoid sampling scheme to handle one-point
bandit feedback.

\item We establish an expected regret bound for the proposed algorithm. The
bound separates the standard contribution from the smooth convex components
and the additional error induced by the global perturbation budget $C$,
thereby making the dependence on $C$ explicit.

\end{itemize}
\section{Related Work}

\paragraph{Bandit convex optimization.}
Bandit convex optimization studies online decision making with convex losses
under limited feedback, where the learner observes only the value of the loss at
the chosen action. Flaxman et al.~\cite{flaxman2005online} introduced a
randomized smoothing method and obtained an expected regret bound of order
\(O(T^{3/4})\) for adversarial convex losses with one-point feedback.
Subsequent work improved the dependence on the time horizon and dimension under
additional assumptions or for special loss classes. For example, Abernethy
et al.~\cite{abernethy2008competing} studied bandit linear optimization using
self-concordant barriers, while Hazan and Levy~\cite{hazan2014bandit} developed
algorithms toward near-optimal regret bounds for bandit convex optimization.
Kernel-based methods were later used to obtain improved guarantees for bandit
convex optimization over general convex bodies~\cite{bubeck2021kernel}. These
works assume that the observed loss functions are convex, whereas our setting
allows the observed losses to be nonconvex and nonsmooth due to bounded
cumulative perturbations.

\paragraph{Smooth convex losses and function-value feedback.}
Smoothness assumptions have also been used to improve regret guarantees in
bandit convex optimization. Saha and Tewari~\cite{saha2011improved} studied
online smooth convex optimization with bandit feedback and obtained improved
regret guarantees compared with the general nonsmooth convex case. Dekel
et al.~\cite{dekel2015bandit} further analyzed bandit smooth convex
optimization by refining the bias--variance tradeoff in gradient estimation.
Stronger guarantees are possible when the learner can query each loss at more
than one point. Agarwal et al.~\cite{agarwal2010optimal} studied multi-point
bandit feedback, and Shamir~\cite{shamir2017optimal} gave optimal guarantees for
bandit convex optimization with two-point feedback. These works
are complementary to ours: they exploit exact convexity, smoothness, or stronger
feedback models, while we study one-point bandit feedback when the observed
losses are only approximately smooth and convex.

\paragraph{Stochastic, nonconvex, and perturbed bandit models.}
Our setting is different from classical stochastic bandit models, where losses
or rewards are typically generated from fixed distributions and the objective is
to learn actions with favorable expected performance
\cite{bubeck2012regret,lattimore2020bandit}. We instead consider an adversarial
loss sequence, and the perturbations in our model need not be independent or
zero-mean. Our work is also distinct from fully general nonconvex bandit
optimization. In that setting, standard static regret is often too strong, and
existing work typically considers alternative criteria such as local regret,
dynamic regret, or stationarity-based measures
\cite{heliou2020online,guan2023online}. Closest to our work is the study of
approximately linear losses by Cheng et al.
\cite{cheng2025approximately,cheng2026adversarial}, where each loss is a linear
function plus a bounded perturbation. In contrast, we consider approximately
smooth convex loss sequences, where the structured component is a general
smooth convex function.

\begin{table}[t]
\centering
\caption{Comparison with related bandit optimization settings. Constants and
dimension factors are omitted.}
\label{tab:comparison}
\scriptsize
\setlength{\tabcolsep}{4pt}
\renewcommand{\arraystretch}{1.15}
\begin{tabularx}{\linewidth}{|>{\raggedright\arraybackslash}p{0.38\linewidth}|>{\centering\arraybackslash}p{0.18\linewidth}|>{\raggedright\arraybackslash}X|}
\hline
\textbf{Loss class} & \textbf{Feedback} & \textbf{Regret} \\
\hline
Convex Lipschitz \cite{flaxman2005online}
& One-point
& \(\widetilde{\mathcal{O}}(T^{3/4})\) \\
\hline
Smooth convex \cite{saha2011improved,dekel2015bandit}
& One-point
& \(\widetilde{\mathcal{O}}(T^{2/3})\) \\
\hline

\hline
Convex \cite{agarwal2010optimal,shamir2017optimal}
& Two-/multi-point
& \(\widetilde{\mathcal{O}}(\sqrt{T})\);
  \(\Omega(\sqrt{T})\) lower bound \cite{bubeck2015bandit} \\
\hline

Approximately linear \cite{cheng2025approximately,cheng2026adversarial}
& One-point
& \(\widetilde{\mathcal{O}}(\sqrt{T}+\sqrt{CT})\) \\
\hline
This work: Approximately smooth convex
& One-point
& \(\widetilde{\mathcal{O}}(T^{2/3}+\sqrt{C}T^{2/3})\) \\
\hline
\end{tabularx}
\end{table}

\section{Preliminaries}
This section introduces the necessary notation and presents the key definitions used throughout the paper. Then we give our problem setting.

\subsection{Notation}
We abbreviate the $2$-norm $\| \cdot \|_2$ as $\|\cdot\|$. 
For a twice differentiable convex function $\mathcal{R}: \Real^d \to \Real$ and 
any $x, h\in \Real^d$, let  
$\| h \|_{x}=
\| h \|_{\nabla^{2}\mathcal{R}(x)}=\sqrt{h^{\top}\nabla^{2}\mathcal{R}(x)h}$, 
and 
$\| h \|_{x}^{*}=
\| h \|_{(\nabla^{2}\mathcal{R}(x))^{-1}}=\sqrt{h^{\top}(\nabla^{2}\mathcal{R}(x))^{-1}h}$, 
respectively.

Let
$
\mathcal{K}
=
\{x \in \mathbb{R}^d : \|x\|_2 \le D\}
$
be the Euclidean ball centered at the origin with radius \(D \ge 1\).
For any \(\delta \in (0,1)\), define the shrunk set
$
\mathcal{K}_{\delta}
=
\left\{
x \in \mathbb{R}^{d}
\;\middle|\;
\frac{1}{1-\delta}x \in \mathcal{K}
\right\}.
$
Let $\mathbb{S}^d_{1}=\{x\mid\| x \| = 1\}$ and $\mathbb{B}^d_{1}=\{x\mid\| x \| \leq 1\}$.
$\mathbb{E}_t[\cdot]$ denotes the conditional expectation
given the history up to round $t-1$.

\subsection{$\nu$-self-concordant barrier}
We next introduce the notion of a $\nu$-self-concordant barrier, which provides a convenient way to characterize the geometry of the feasible set, particularly near its boundary. The Hessian of the barrier induces a local metric that will be used in the subsequent analysis and algorithmic construction. We first recall the relevant definition.

\begin{definition}
    \label{def2}
     Let $\mathcal{K}\subseteq\Real^{d}$ be a convex set with a nonempty interior $int(\mathcal{K})$. A function $\mathcal{R}:\operatorname{int}(\mathcal K) \to \Real$ is called a $\nu$-self-concordant barrier on $\mathcal{K}$ if 
    \begin{enumerate}
        \item $\mathcal{R}$ is three times continuously differentiable and convex and approaches infinity along any sequence of points approaching the boundary of $\mathcal{K}$. 
        \item For every $h\in\Real^{d}$ and $x\in int(\mathcal{K})$ the following holds: 
    \begin{equation}
        |\sum_{i=1}^{d}\sum_{j=1}^{d}\sum_{k=1}^{d}\frac{\partial^{3}\mathcal{R}(x)}{\partial x_{i}\partial x_{j}\partial x_{k}}h_{i}h_{j}h_{k}|\leq 2 \| h \|^{3}_{x},
    \end{equation}
    \begin{equation}
        |\nabla \mathcal{R}(x)^{\top}h |\leq \sqrt{\nu}\| h \|_{x},
    \end{equation}

    \end{enumerate}
\end{definition}

\subsection{Convexity, Smoothness, and Lipschitz Continuity}

In this paper, we frequently use the notions of convexity,
smoothness, and Lipschitz continuity.
For completeness, we briefly recall their definitions.

\begin{definition}[Convex Function]
A function
$f : \mathcal{K} \to \mathbb{R}$
defined on a convex set
$\mathcal{K} \subseteq \mathbb{R}^d$
is called convex if for all
$x,y \in \mathcal{K}$
and all
$\lambda \in [0,1]$,
\[
f(\lambda x + (1-\lambda)y)
\le
\lambda f(x) + (1-\lambda)f(y).
\]
Equivalently, if $f$ is differentiable, then
\[
f(y)
\ge
f(x)
+
\nabla f(x)^\top (y-x),
\qquad
\forall x,y \in \mathcal{K}.
\]
\end{definition}

\begin{definition}[$\beta$-smooth Function]
A differentiable function
$f : \mathcal{K} \to \mathbb{R}$
is called
$\beta$-smooth
if for all
$x,y \in \mathcal{K}$,
\[
f(y)
\le
f(x)
+
\nabla f(x)^\top (y-x)
+
\frac{\beta}{2}\|x-y\|_2^2.
\]
\end{definition}

\begin{definition}[$G$-Lipschitz Continuity]
A function
$f : \mathcal{K} \to \mathbb{R}$
is said to satisfy
$G$-Lipschitz continuity
if for all
$x,y \in \mathcal{K}$,
\[
|f(x)-f(y)|
\le
G\|x-y\|_2.
\]
\end{definition}

\subsection{Approximately Convex and Smooth Function Sequences}

We now introduce the main function class considered in this paper, namely,
approximately convex and smooth function sequences.

\begin{definition}[$C$-approximately convex and $\beta$-smooth function sequence]
\label{def1}
A sequence of functions $\{f_t\}_{t=1}^T$ with
$f_t : \mathcal{K} \to \mathbb{R}$
is called
\emph{$C$-approximately convex and $\beta$-smooth}
if there exists a sequence of convex and $\beta$-smooth functions
$\{h_t\}_{t=1}^T$ with
$h_t : \mathcal{K} \to \mathbb{R}$
such that
\[
\sup_{x_1,\ldots,x_T \in \mathcal{K}}
\sum_{t=1}^T
\left|
f_t(x_t)-h_t(x_t)
\right|
\le C.
\]
\end{definition}

\noindent
For convenience, define
\[
\sigma_t(x)
=
f_t(x)-h_t(x).
\]
Then,
\[
\sup_{x_1,\ldots,x_T\in\mathcal K}
\sum_{t=1}^T
|\sigma_t(x_t)|
\le C.
\]

This definition allows the perturbation functions $\sigma_t$ to be
nonconvex and nonsmooth, without imposing any stochastic assumptions on
them. Consequently, the loss functions $f_t$ themselves may also be
nonconvex or nonsmooth. The parameter $C$ quantifies the cumulative
deviation of the sequence $\{f_t\}_{t=1}^T$ from the underlying sequence
of convex and $\beta$-smooth functions $\{h_t\}_{t=1}^T$. In particular,
when $C=0$, we have $f_t=h_t$ for every $t$, and the setting reduces to
the standard case of convex and $\beta$-smooth loss functions.

\subsection{Problem Setting}

We consider an online bandit optimization problem over the Euclidean ball
$\Kset\subseteq\mathbb{R}^d$ centered at the origin. The loss at round $t$ is of the form
$
f_t(x)=h_t(x)+\sigma_t(x),
$
where $h_t:\Kset\to\mathbb{R}$ is convex and $\beta$-smooth with respect to the Euclidean norm. We further assume that $h_t$ is $G$-Lipschitz Continuity.

The sequence $\{h_t\}_{t=1}^T$ is selected by an oblivious adversary
before the interaction begins. 
In contrast, the perturbation function $\sigma_t$ may be selected
adaptively after the adversary observes the player's current action
$x_t$. Despite this adaptivity, the cumulative bound on the
perturbations in Definition~\ref{def1} is required to hold pathwise and
uniformly over all $z_1,\ldots,z_T\in\Kset$, not merely along the
realized action sequence $x_1,\ldots,x_T$.
We further assume that the resulting losses are uniformly bounded:
\[
|f_t(x)|\leq L,
\qquad
\forall t\in[T],\quad \forall x\in\Kset.
\]
Before the interaction begins, an oblivious adversary selects a sequence $\{h_t\}_{t=1}^T$ of convex and $\beta$-smooth functions. The sequence is fixed in advance but is not revealed to the player. At each round $t\in[T]$: \begin{enumerate} 
    \item The player selects an action $x_t\in\Kset$. 
    
    \item After observing $x_t$, the adversary selects a perturbation function $\sigma_t:\Kset\to\mathbb{R}$, subject to the cumulative bound specified above. 
    
    \item The player incurs the loss $ f_t(x_t) = h_t(x_t)+\sigma_t(x_t) $ and observes only the scalar feedback $f_t(x_t)$. 
\end{enumerate}

The performance of the player is measured by the regret with respect to
the best fixed action in hindsight:
\[
\operatorname{Reg}_T
:=
\sum_{t=1}^T f_t(x_t)
-
\min_{x\in\Kset}
\sum_{t=1}^T f_t(x).
\]
The goal is to design a bandit algorithm whose regret is sublinear in
$T$.

\section{Main Results}
In this section, we first introduce SCRiBLe \cite{abernethy2008competing}, followed by presenting the main contributions of this paper with detailed explanations.

\subsection[SCRiBLe]{SCRiBLe \cite{abernethy2008competing}}

Bandit linear optimization (BLO) can be viewed as a special case of bandit
convex optimization (BCO), in which each loss function is linear, namely,
$
f_t(x)=\theta_t^\top x
$
for some loss vector $\theta_t\in\mathbb{R}^d$. In the more general BCO setting,
the loss functions may be nonlinear and convex, and the learner observes only
the scalar value of the loss at the played point. Nevertheless, SCRiBLe\cite{abernethy2008competing} provides a useful foundation for our BCO algorithm,
as its local exploration scheme and bandit gradient estimator can be extended
from linear losses to smooth convex losses.
Motivated by this connection, we briefly revisit SCRiBLe\cite{abernethy2008competing}, a fundamental algorithm for bandit linear
optimization.

SCRiBLe utilizes a $\nu$-self-concordant barrier $\mathcal{R}$ defined on the
action set $\mathcal{K}$.
At each round $t$, the player selects a point $x_t \in \mathcal{K}$, while the
actual played point is
$
y_t = x_t + \mathbf{A}_t \mu_t,
$
where
$
\mathbf{A}_t = [\nabla^2 \mathcal{R}(x_t)]^{-1/2},
$
and $\mu_t$ is sampled uniformly at random from the $d$-dimensional unit sphere
$\mathbb{S}_1^d$.
The learner then observes the scalar bandit feedback
$
f_t(y_t) := \theta_t^\top y_t,
$
where $\theta_t \in \mathbb{R}^d$ is the loss vector at round $t$.
Using only the scalar feedback $f_t(y_t)$, SCRiBLe constructs the unbiased estimator
$
g_t = d f_t(y_t)\mathbf{A}_t^{-1}\mu_t
$
of the loss vector $\theta_t$.
The algorithm maintains a sequence $\{x_t\}_{t=1}^T \subset \mathcal{K}$ updated by
$
x_{t+1}
=
\arg\min_{x \in \mathcal{K}}
\left\{
\eta \sum_{\tau=1}^t g_\tau^\top x + \mathcal{R}(x)
\right\},
$
where $\eta$ is the learning rate.

\begin{algorithm}[h]
    \caption{Shrunk and Scaled SCRiBLe (SS-SCRiBLe)}
    \label{alg1}
    \begin{algorithmic}[1]
        \Require
        
        $T$, parameters $\eta\in\Real, \delta\in(0, 2/3]$, $b\in(0, 1]$,
        and a $2$-self-concordant barrier $\mathcal{R}(x)=-\log(D^2-\|x\|^2)$.
        \State Initialize: $x_1=\arg\min\limits_{x\in\mathcal{K_{\delta}}}{\mathcal{R}(x)}$
        \For{$t=1,..,T$}
            \State let $\mathbf{A}_{t}=b[\nabla^2\mathcal{R}(x_{t})]^{-\frac{1}{2}}$ \State  Draw $\mu_{t}$ from $\mathbb{S}_{1}^{d}$ uniformly, set $y_t=x_{t}+\mathbf{A}_{t}\mu_{t}$.  
            \State Play $y_{t}$, observe and incur loss $f_{t}(y_{t})$. Let $g_{t}=df_{t}(y_{t})\mathbf{A}_{t}^{-1}\mu_{t}$.
            \State Update $x_{t+1}=\arg\min\limits_{x\in\mathcal{K_{\delta}}}{\eta\sum_{\tau=1}^{t}g_{\tau}^{\top}x+\mathcal{R}(x)}$
        \EndFor
    \end{algorithmic}
\end{algorithm}
For the decision set $\mathcal{K}$, we apply Shrunk and Scaled SCRiBLe
(SS-SCRiBLe), presented in Algorithm~\ref{alg1}, to a sequence of
$C$-approximately convex and $\beta$-smooth loss functions.

\paragraph{Difference from SCRiBLe.}
The proposed SS-SCRiBLe algorithm is based on the SCRiBLe framework,
which exploits the local geometry induced by a self-concordant barrier for
random exploration. The key difference is that SS-SCRiBLe updates the
internal point over a shrunk feasible set $\mathcal{K}_{\delta}$, rather than
over the original set $\mathcal{K}$:
$x_{t+1}
=
\arg\min_{x\in\mathcal{K}_{\delta}}
\left\{
\eta\sum_{\tau=1}^{t} g_{\tau}^{\top}x+\mathcal{R}(x)
\right\}$.
It also uses a scaled exploration matrix
$\mathbf{A}_t
=
b[\nabla^2\mathcal{R}(x_t)]^{-1/2},
b\in(0,1)$,
so that the played action
$y_t=x_t+\mathbf{A}_t\mu_t$
remains inside $\mathcal{K}$; see Lemma~\ref{lemm:normal-property}.
Thus, compared with the original SCRiBLe algorithm, SS-SCRiBLe makes both
the update and the exploration step more conservative.

This modification is necessary because the original SCRiBLe analysis relies on
the linearity of the losses. When $f_t(y_t)=\theta_t^\top y_t$, the estimator
$g_t$ is unbiased for $\theta_t$. In our setting, even in the absence of the additive perturbation \(\sigma_t\),
a nonlinear smooth convex loss $h_t$ does not yield an unbiased estimator of
$\nabla h_t(x_t)$. Moreover, when
$f_t=h_t+\sigma_t$, the estimator
$g_t
=
d(h_t(y_t)+\sigma_t(y_t))\mathbf{A}_t^{-1}\mu_t$
contains an additional perturbation-dependent component. Hence the regret
analysis must control error terms caused by both the nonlinearity of $h_t$ and
the perturbation $\sigma_t$.

The perturbation term also creates additional difficulties in the regret
analysis. Even when the cumulative perturbation is bounded, its effect may be
amplified by the local geometry of the self-concordant barrier. If $x_t$
approaches the boundary of $\mathcal{K}$, the Hessian
$\nabla^2\mathcal{R}(x_t)$ may become large
\cite{nemirovski2004interior}, making the perturbation-dependent errors hard to
control uniformly. Together, the shrinkage and scaling
modifications allow us to control the interaction between the local barrier
geometry, the nonlinear losses, and the perturbation terms in the regret
analysis.

We now present our main result, an expected regret bound for the problem.

\begin{theorem}
\label{Theo:expected-regret}
Let $\delta$ be any number in $(0, 2/3]$, $w\in\arg\min\limits_{x \in \mathcal{K}} \sum_{t=1}^{T} f_t(x)$, and set $\gamma = \max\left\{\delta,\frac{1}{T}\right\}$. 
Then SS-SCRiBLe guarantees the following expected regret bound:
\begin{equation}
\begin{aligned}
\mathbb{E}\!\left[\sum_{t=1}^{T}f_t(y_t)
-
\sum_{t=1}^{T} f_t(w)\right]
&\leq
\frac{b^{2}\beta D^2T}{2}
+2\eta b^{-2} Td^2L^2
+\frac{2\log\frac{1}{\gamma}}{\eta}
\\
&\quad
+d b^{-1}C(4+4\sqrt{2})
\left(\frac{1}{\delta}-1\right)
+\gamma TGD
+2C.
\end{aligned}
\end{equation}
Moreover, if
$
\delta = \frac{\sqrt{CdD}}{\sqrt{T}}
\leq \frac{2}{3}$, setting $ 
\eta = T^{-\frac{2}{3}}d^{-1}D^{-1}L^{-1}
\sqrt{\log \frac{1}{\gamma}},
b = T^{-\frac{1}{6}}D^{-1},
$
we obtain
\begin{equation}
 \begin{aligned}
\mathbb{E}\!\left[
    \sum_{t=1}^{T} f_t(y_t)
    -
    \sum_{t=1}^{T} f_t(w)
\right]
\leq{}&
\frac{\beta T^{2/3}}{2}
+4T^{2/3}dDL\sqrt{\log T}+
(4+4\sqrt{2})T^{2/3}\sqrt{CdD}
\\
&
+\max\{1,\sqrt{TCdD}\}\,GD
+2C.
\end{aligned}
\end{equation}
\end{theorem}

\section{Technical Lemmas and Proof of the Main Theorem}
This section introduces several essential lemmas and presents the proofs of the main theorems.
Omitted proofs are shown in the arxiv version\footnote{https://arxiv.org/abs/2606.19891}.

\subsection{Useful lemmas}
In addition to the properties of the $\nu$-self-concordant barrier and its related lemmas, we further present several auxiliary lemmas.

\begin{lemma}[\cite{nemirovski2004interior}]
\label{lemm:normal-property}
    If $\mathcal{R}$ is a $\nu$-self-concordant barrier on $\mathcal{K}$, then the Dikin ellipsoid centered at $x\in int(\mathcal{K})$, defined as $\{y: \| y-x\|_{x}\leq 1\}$, is always within $\mathcal{K}$. 
\end{lemma}
By Lemma~\ref{lemm:normal-property}, the played action $y_t$ in
SS-SCRiBLe remains in $\mathcal{K}$ whenever $b\leq 1$.

\begin{lemma}[\cite{hazan2016introduction}]
\label{lemm:normal-log-property}
        Let $\mathcal{R}$ be a $\nu$-self-concordant barrier over $\mathcal{K}$, then for all $x, z\in int(\mathcal{K}): \mathcal{R}(z)-\mathcal{R}(x)\leq \nu \ln \frac{1}{1-\pi_{x}(z)}$, where $\pi_{x}(z)=\inf\{t\geq 0 : x+t^{-1}(z-x)\in\mathcal{K}\}$.
\end{lemma}

\begin{lemma}
\label{lemm:control-h}
    Assume $\delta\leq \frac{2}{3}$. For SS-SCRiBLe and for any $x,y \in \mathcal{K_{\delta}}$, it holds that
    \begin{equation}
        \| y-x \|_{x}
        \leq 2(\frac{1}{\delta}-1)
        \left(\nu+2\sqrt{\nu}\right).
    \end{equation}
\end{lemma}
\noindent
The above two lemmas will be used in the proof of
Theorem~\ref{Theo:expected-regret}; see
inequalities~\ref{eq:useLemm2} and~\ref{eq:bound-x-distance}.

\begin{lemma}
\label{lemm:close-property}
Let $\mathcal R$ be a $\nu$-self-concordant barrier on $\mathcal K_\delta$.
Let $x_t$ and $x_{t+1}$ be the points generated in SS-SCRiBLe. If
$
    \eta \|g_t\|_{x_t}^{*}\le \frac12,
$
then
\[
    \|x_{t+1}-x_t\|_{x_t}
    \le
    2\eta\|g_t\|_{x_t}^{*}.
\]
\end{lemma}

\noindent
This result will help us bound $g_{t}^{\top} (x_{t}-h)$, where $h\in\mathcal{K}$ (see inequality~\ref{eq:uselemmboundxt}).

\begin{lemma}
\label{lemm:FTRL}
        For SS-SCRiBLe and for every $h\in \mathcal{K}$,
        $\sum_{t=1}^{T}g_{t}^{\top}x_{t}-\sum_{t=1}^{T}g_{t}^{\top}h \leq \sum_{t=1}^{T}[g_{t}^{\top}x_{t}-g_{t}^{\top}x_{t+1}]+\frac{1}{\eta}[\mathcal{R}(h)-\mathcal{R}(x_{1})]$.
\end{lemma}
\noindent
Since the update
$x_{t+1}=\arg\min\limits_{x\in\mathcal{K}_{\delta}} {\eta\sum_{\tau=1}^{t}g_{\tau}^{\top}x+\mathcal{R}(x)}$
is exactly the FTRL update over $\mathcal{K}_{\delta}$ with linear losses $g_t^\top x$, Lemma 5.3 of Hazan~\cite{hazan2016introduction} can be applied directly to SS-SCRiBLe.

\begin{lemma}
[\cite{hazan2014bandit}]
\label{lemm:ellipsoidal-smoothing}
Let $f:\mathbb{R}^d\to\mathbb{R}$ be a continuous function, let
$A\in\mathbb{R}^{d\times d}$ be an invertible matrix, and let
$v\sim \mathbb{B}^d$ and $u\sim \mathbb{S}^d$ be uniformly distributed over
the unit ball and the unit sphere, respectively. Define the smoothed version
of $f$ with respect to $A$ by
\[
    \widehat f(x)
    =
    \mathbb{E}_{v}\bigl[f(x+Av)\bigr].
\]
Then $\widehat f$ is differentiable and satisfies
\[
    \nabla \widehat f(x)
    =
    \mathbb{E}_{u}\bigl[
        d f(x+Au) A^{-1}u
    \bigr].
\]
Moreover, if $A\succ 0$, the following properties hold:
\begin{enumerate}
    \item If $f$ is convex, then $\widehat f$ is also convex.

    \item If $f$ is convex and $\beta$-smooth, and $\lambda_{\max}$ denotes
    the largest eigenvalue of $A$, then for every $x\in\mathbb{R}^n$,
    \[
        0
        \le
        \widehat f(x)-f(x)
        \le
        \frac{\beta}{2}\|A^2\|_2
        =
        \frac{\beta}{2}\lambda_{\max}^2 .
    \]
\end{enumerate}
The second property also holds if the smoothed version of $f$ is defined by
\[
    \widehat f(x)
    =
    \mathbb{E}_{u\sim\mathbb{S}^d}\bigl[f(x+Au)\bigr],
\]
that is, by averaging the original function values over the unit sphere
rather than the unit ball. The proof is similar.
\end{lemma}

\begin{lemma}
\label{lemm:barrier-hessian-inverse-bound}
Consider the 2-self-concordant barrier
$
    \mathcal{R}(x)
    =
    -\log\left(D^2-\|x\|^2\right)
$
for the Euclidean ball
$
    \mathcal{K}
    =
    \{x\in\mathbb{R}^d:\|x\|<D\}.
$
For any $x\in\mathcal{K}$ and $b>0$, define
$
    A
    =
    b\left(\nabla^2\mathcal{R}(x)\right)^{-1/2}.
$
Then
\[
    \left\|A^2\right\|_2
    \le
    \frac{b^2D^2}{2}.
\]
\end{lemma}

\noindent
With the help of Lemmas above, we are ready to present the proof of Theorem~\ref{Theo:expected-regret}.

\subsection{Proof}

\begin{proof} 
Let $\gamma \in [\delta,1)$, and define the shrunk set
$
\mathcal{K}_{\gamma} = \left\{ x \in \mathbb{R}^d \,\middle|\, \frac{1}{1-\gamma} x \in \mathcal{K} \right\}.
$
It holds that
$
\mathcal{K}_{\gamma} \subseteq \mathcal{K}_{\delta} \subset \mathcal{K}.
$
Let
\(
w \in \arg\min\limits_{x \in \mathcal{K}} \sum_{t=1}^{T} f_t(x),
\)
and define
\(
w_\gamma=(1-\gamma)w.
\)
Then,
    \begin{equation}
       \| w-w_{\gamma}\|
        = \| w-(1-\gamma)w\|
        \leq \gamma D.
        \label{eq:boundwgamma}
    \end{equation}

    Recall that each loss function can be written as $f_t(x)=h_t(x)+\sigma_t(x)$.
    We decompose the expected regret as
    \begin{align}
    \mathbb{E}\!\left[
    \sum_{t=1}^{T} f_t(y_t) - \sum_{t=1}^{T} f_t(w)
    \right]
    &=
    \mathbb{E}\!\left[
    \sum_{t=1}^{T} \bigl( h_t(y_t) + \sigma_t(y_t) \bigr)
    -
    \sum_{t=1}^{T} \bigl( h_t(w)  + \sigma_t(w) \bigr)
    \right]
    \nonumber\\
    &=
    \mathbb{E}\!\left[
    \sum_{t=1}^{T} h_t(y_t)
    -
    \sum_{t=1}^{T} h_t(w)
    \right]
    +
    \mathbb{E}\!\left[
    \sum_{t=1}^{T} \sigma_t(y_t)
    -
    \sum_{t=1}^{T} \sigma_t(w)
    \right].
    \end{align}

    We first bound the convex and smooth component. Observe that
    \begin{align}
    \mathbb{E}\!\left[
    \sum_{t=1}^{T} h_t(y_t)
    -
    \sum_{t=1}^{T} h_t(w)
    \right]
    &=
    \mathbb{E}[\sum_{t=1}^{T} h_t(y_t)
    -
    \sum_{t=1}^{T} h_t(w_{\gamma})]
    +
    \mathbb{E}[\sum_{t=1}^{T} h_t(w_{\gamma})
    -
    \sum_{t=1}^{T} h_t(w)],
    \end{align}
    where $w_{\gamma} = (1-\gamma)w$.

\noindent
    By the $G$-Lipschitz continuity of $h_t$ and
    \eqref{eq:boundwgamma}, we have
    \begin{align}
    \mathbb{E}[
   \sum_{t=1}^{T} h_t(w_{\gamma})
    -
    \sum_{t=1}^{T} h_t(w)
    ]
    \le
    \mathbb{E}[\sum_{t=1}^{T}G\|w-w_{\gamma}\| ]
    \le \gamma DGT.
    \end{align} 

\noindent
   Then,
    \begin{align}
    \mathbb{E}\left[\sum_{t=1}^{T} 
    h_t(y_t)- \sum_{t=1}^{T}h_t(w_{\gamma})
    \right]
    &=
    \mathbb{E}\sum_{t=1}^{T} 
    [
    h_t(y_t)-h_t(x_t)
    ]
    \\
    &+
    \mathbb{E}\sum_{t=1}^{T} 
    [
    h_t(x_t)-\hat{h}_t(x_t)
    ]
    \\
    &+
    \mathbb{E}\sum_{t=1}^{T} 
    [
    \hat{h}_t(w_{\gamma})-h_t(w_{\gamma})
    ]
    \\
    &+
    \mathbb{E}\sum_{t=1}^{T} 
    [
    \hat{h}_t(x_t)-\hat{h}_t(w_{\gamma})
    ]
    \label{eq:d_of_hath}
    \end{align} 
    where $\hat{h}_t$ is defined as
    $
    \hat{h}_t(x)
    =
    E\!\left[h_t(x+A_t v)\right],
    \forall x \in \mathcal{K}.
    $

    Here $v$ is uniformly distributed over the unit ball $B^d$, and $A_t$ is defined as in Algorithm~\ref{alg1}.

\noindent
   From Lemma~\ref{lemm:ellipsoidal-smoothing}, we obtain
    \begin{flalign}
    \hspace{0.5em}
    \mathbb{E}
    \left[
    h_t(y_t)-h_t(x_t)
    \right] 
    =
    \mathbb{E}
    \left[
    \mathbb{E}_{u_t}
    \left[
    h_t(x_t + A_t u_t)
    \right]
    -
    h_t(x_t)
    \mid x_t
    \right]
    \leq
    \frac{\beta}{2}
    \mathbb{E}
    \left\|A_t^2\right\|_2. &&
    \end{flalign}
    
    \begin{flalign}
    \hspace{0.5em}
    \mathbb{E}
    \left[
    \hat{h}_t(w_{\gamma})-h_t(w_{\gamma})
    \right] 
    =
    \mathbb{E}
    \left[
    \mathbb{E}
    \left[
    \hat{h}_t(w_{\gamma})-h_t(w_{\gamma})
    \mid x_t
    \right]
    \right]
    \leq
    \frac{\beta}{2}
    \mathbb{E}
    \left\|A_t^2\right\|_2. &&
    \end{flalign}
    
    \begin{flalign}
    \hspace{0.5em}
    \mathbb{E}
    \left[
    h_t(x_t)-\hat{h}_t(x_t)
    \right]
    \leq 0. &&
    \end{flalign}

    \noindent
    where
    $
    A_t = b\left[\nabla^2 \mathcal{R}(x_t)\right]^{-\frac{1}{2}} .
    $
    
\noindent
By Lemma~\ref{lemm:barrier-hessian-inverse-bound}, we know 
\begin{equation}
    \sum_{t=1}^T \frac{\beta}{2}
    \mathbb{E}
    \left\|A_t^2\right\|_2\le  \frac{\beta b^2D^2T}{4}
\end{equation}
Therefore, 
\begin{align}
    \mathbb{E}\sum_{t=1}^{T} 
    [
    h_t(y_t)-h_t(x_t)
    ]
    +
    \mathbb{E}\sum_{t=1}^{T} 
    [
    h_t(x_t)-\hat{h}_t(x_t)
    ]
    +
    \mathbb{E}\sum_{t=1}^{T} 
    [
    \hat{h}_t(w_{\gamma})-h_t(w_{\gamma})
    ]
    \le \frac{\beta b^2D^2T}{2}
\end{align}

\noindent  
We now bound the last term in equation~(\ref{eq:d_of_hath}).
 Using Lemma~\ref{lemm:ellipsoidal-smoothing}, we know $\mathbb{E}_{t}[g_{t}]
    =
    \mathbb{E}_{t}\!\left[
    d\,(h_t(y_t)+\sigma_{t}(y_{t}))\mathbf{A}_{t}^{-1}\mu_{t}
    \right]
    =
    \nabla \hat{h}_t(x_t)
    +
    \mathbb{E}_{t}\!\left[
    d\,\sigma_{t}(y_{t})\mathbf{A}_{t}^{-1}\mu_{t}
    \right]$.
    Let $M_{t}=\duoE_{t}[d\sigma_{t}(y_{t})\mathbf{A}_{t}^{-1}\mu_{t}]$. Then $\nabla \hat{h}_t(x_t)=\duoE_{t}[g_{t}]-M_{t}$. Consequently,
    \begin{align}
        \mathbb{E}\sum_{t=1}^{T} 
        [\hat{h}_t(x_t)-\hat{h}_t(w_{\gamma})]
        &\le \duoE\sum_{t=1}^{T}\nabla \hat{h}_t(x_t)^{\top}(x_{t}-w_{\gamma})
           \nonumber \\
        &= \mathbb{E}\sum_{t=1}^{T}[\duoE_{t}[g_{t}]-M_{t}]^{\top}(x_{t}-w_{\gamma})
           \nonumber \\
        &= \mathbb{E}\sum_{t=1}^{T}\duoE_{t}[g_{t}]^{\top}(x_{t}-w_{\gamma})
           +\mathbb{E}\sum_{t=1}^{T}M_{t}^{\top}(w_{\gamma}-x_{t}).
    \label{eq:divideHath}
    \end{align}
The first term is bounded by Lemma~\ref{lemm:FTRL} as

\begin{align}
\nonumber
\duoE\sum_{t=1}^{T}\duoE_{t}[g_{t}]^{\top}(x_{t}-w_{\gamma})
&= \duoE\!\left[\sum_{t=1}^{T} g_{t}^{\top}(x_{t}-w_{\gamma})\right] \\
\nonumber
&\le \duoE\!\left[\sum_{t=1}^{T}
      \bigl(g_{t}^{\top}x_{t}-g_{t}^{\top}x_{t+1}\bigr)
      +\frac{1}{\eta}\bigl(\mathcal{R}(w_{\gamma})-\mathcal{R}(x_{1})\bigr)
      \right] \\
&\le \duoE\!\left[\sum_{t=1}^{T}
      \| g_{t} \|^{*}_{x_{t}} \, \| x_{t}-x_{t+1} \|_{x_{t}}
      \right]
      +\frac{1}{\eta}\bigl(\mathcal{R}(w_{\gamma})-\mathcal{R}(x_{1})\bigr).
      \label{eq:usinglemmFTRL}
\end{align}
Lemma~\ref{lemm:close-property} implies that
$
\| x_{t+1} - x_t \|_{x_t} \le 2\eta \| g_t \|^{*}_{x_t}
$
by the choice of $\eta$.
Moreover,
\begin{align}
\| g_t \|^{*}_{x_t}
&=
\bigl\| d f_t(y_t)\mathbf{A}_t^{-1}\mu_t \bigr\|^{*}_{x_t}
\\
&=
\sqrt{
d^{2} f_t^{2}(y_t)b^{-2}\,
\mu_t^{\top}
\nabla^{2}\mathcal{R}(x_t)^{1/2}
\bigl(\nabla^{2}\mathcal{R}(x_t)\bigr)^{-1}
\nabla^{2}\mathcal{R}(x_t)^{1/2}
\mu_t
}
\\
&\le Ldb^{-1}.
\end{align}
where the last inequality follows from the assumption
$|f_t(y_t)| \le L$.

\noindent
Therefore,
    \begin{equation}
    \| g_{t} \|^{*}_{x_{t}} \, \| x_{t}-x_{t+1} \|_{x_{t}}
    \leq 2\eta L^{2}d^{2}b^{-2},
    \label{eq:uselemmboundxt}
    \end{equation}
    which implies
    \[
    \mathbb{E}\!\left[\sum_{t=1}^{T}
    \| g_{t} \|^{*}_{x_{t}} \, \| x_{t}-x_{t+1} \|_{x_{t}}
    \right]
    \leq 2\eta L^{2}d^{2}b^{-2}T.
    \]

\noindent
We bound the last term on the right-hand side of
inequality~\ref{eq:usinglemmFTRL} by using
Lemma~\ref{lemm:normal-log-property}:
    \begin{equation}
        \frac{1}{\eta}(\mathcal{R}(w_{\gamma})-\mathcal{R}(x_{1}))
                \leq \frac{\nu \log (\frac{1}{\gamma})}{\eta}.
    \label{eq:useLemm2}
    \end{equation}

\noindent
For the second term of~\eqref{eq:divideHath}, by Cauchy–Schwarz inequality,
    \begin{align}
    \sum_{t=1}^{T} M_t^{\top}(w_{\gamma} - x_t)
    &\le \sum_{t=1}^{T} \| M_t \|_{x_t}^{*} \, \| w_{\gamma} - x_t \|_{x_t}. 
    \label{eq:boundM0}
    \end{align}
    We next bound $\| M_t \|_{x_t}^{*}$. By the definition of the dual norm,
    \begin{align}
    \| M_t \|_{x_t}^{*}
    &= \sqrt{ M_t^{\top} \nabla^{2}\mathcal{R}(x_t)^{-1} M_t } .
    \label{eq:boundM1}
    \end{align}
    Recall that
    \(
    M_t = \mathbb{E}_t\!\left[d\,\sigma_t(y_t)\mathbf{A}_t^{-1}\mu_t\right].
    \)
    Substituting this expression yields
    \begin{align}
    \| M_t \|_{x_t}^{*}
    &= \sqrt{
    \mathbb{E}_t\!\left[d\,\sigma_t(y_t)\mathbf{A}_t^{-1}\mu_t\right]^{\top}
    \nabla^{2}\mathcal{R}(x_t)^{-1}
    \mathbb{E}_t\!\left[d\,\sigma_t(y_t)\mathbf{A}_t^{-1}\mu_t\right]
    } \nonumber\\
    &= \sqrt{
    d^{2}b^{-2}\,
    \mathbb{E}_t[\sigma_t(y_t)\mu_t]^{\top}
    \nabla^{2}\mathcal{R}(x_t)^{1/2}\nabla^{2}\mathcal{R}(x_t)^{-1}\nabla^{2}\mathcal{R}(x_t)^{1/2}
    \mathbb{E}_t[\sigma_t(y_t)\mu_t]
    } \nonumber\\
    &= db^{-1} \, \bigl\| \mathbb{E}_t[\sigma_t(y_t)\mu_t] \bigr\|.
    \end{align} 
    By Jensen's inequality, we further have
    \begin{align}
    \bigl\| \mathbb{E}_t[\sigma_t(y_t)\mu_t] \bigr\|
    &\le \mathbb{E}_t\|\sigma_t(y_t)\, \mu_t\| = \mathbb{E}_t [|\sigma_t(y_t)|\, \|\mu_t\|]\nonumber\\
    &= \mathbb{E}_t|\sigma_t(y_t)|,
    \end{align}
    where we use the fact that $\|\mu_t\|=1$. 
    
    Therefore,
    \begin{align}
    \| M_t \|_{x_t}^{*}
    \le db^{-1}\,\mathbb{E}_t|\sigma_t(y_t)|.
    \label{eq:boundM2}
    \end{align}
By Definition~\ref{def1}, the cumulative perturbation is bounded by $C$
uniformly over all sequences $(y_1,\ldots,y_T)$. Thus,
    \begin{align}
\sum_{t=1}^{T}\mathbb{E}_{t}|\sigma_t(y_t)|
    \leq
    \sum_{t=1}^{T}\max_{y_t}|\sigma_t(y_t)|
    =
    \max_{y_1,\ldots,y_T}
    \sum_{t=1}^{T}|\sigma_t(y_t)| 
    \leq C.
    \label{eq:boundM2partC}
    \end{align}

    We now bound $\| w_{\gamma} - x_t \|_{x_t}$. By Lemma~\ref{lemm:control-h} and the fact that $w_{\gamma}, x_t \in \mathcal{K}_{\delta}$, we have that
    \begin{align}
    \| w_{\gamma} - x_t \|_{x_t}
    \le 2(\nu + 2\sqrt{\nu}) \frac{1-\delta}{\delta}.
    \label{eq:bound-x-distance}
    \end{align}
    
\noindent
Combining the above bound with Eqs.~(\ref{eq:boundM0}, \ref{eq:boundM2}, \ref{eq:boundM2partC}, \ref{eq:bound-x-distance}), we obtain
\[
\mathbb{E}\left[
\sum_{t=1}^{T}M_{t}^{\top}(w_{\gamma}-x_{t})
\right]
\leq  2db^{-1}C(\nu+2\sqrt{\nu})\frac{(1-\delta)}{\delta}.
\]    

\noindent
Finally, since the perturbations are uniformly bounded in the worst case by
Definition~\ref{def1}, we have
\begin{equation}
\mathbb{E}\!\left[
\sum_{t=1}^{T} \sigma_t(y_t)
-
\sum_{t=1}^{T} \sigma_t(w)
\right]
\le 2C .
\label{eq:perturbation-bound}
\end{equation}

\noindent
Combining all the above bounds and $\nu=2$, we obtain
    \[
\begin{aligned}
\mathbb{E}\!\left[\sum_{t=1}^{T}f_t(y_t)
-
\sum_{t=1}^{T} f_t(w)\right]
&\leq
\frac{b^{2}\beta D^2T}{2}
+2\eta b^{-2} Td^2L^2
+\frac{2\log\frac{1}{\gamma}}{\eta}
\\
&\quad
+ db^{-1}C(4+4\sqrt{2})
\left(\frac{1}{\delta}-1\right)
+\gamma TGD
+2C
\\
&\leq
\frac{T^{\frac{2}{3}}\beta}{2}
+
4T^{\frac{2}{3}}dDL\sqrt{\log T}
+
(4+4\sqrt{2})T^{\frac{2}{3}}\sqrt{CdD}
\\
&\quad
+\max\{1,\sqrt{TCdD}\}GD
+
2C.
\end{aligned}
\]

\noindent
where
\[
\begin{aligned}
b &= T^{-\frac{1}{6}}D^{-1},
\qquad
\eta = T^{-\frac{2}{3}}d^{-1}D^{-1}L^{-1}
\sqrt{\log \frac{1}{\gamma}}, \\
\delta &= \frac{\sqrt{CdD}}{\sqrt{T}}
\leq \frac{2}{3},
\qquad
\gamma = \max\left\{\delta,\frac{1}{T}\right\}.
\end{aligned}
\]

\end{proof}
\section{Conclusion}

In this work, we extend adversarial bandit optimization with globally budgeted post-action perturbations from underlying linear losses to underlying convex and $\beta$-smooth losses. Although the resulting losses may be non-convex and non-smooth, the cumulative magnitude of the perturbations is controlled by a global budget. For this setting, we establish expected regret guarantees that explicitly characterize the effect of the perturbation budget. Our analysis modifies a standard bandit optimization algorithm and bounds the additional regret induced by the perturbations. Future work includes improving the parameter dependence of the regret bound and extending the approach to more general structured loss classes.

\begin{credits}
\subsubsection{\ackname}
This work was supported by WISE program (MEXT) at Kyushu
University,  JSPS KAKENHI Grant Numbers JP23K24905,
JP23K28038, and JP24H00685, respectively.

\subsubsection{\discintname}
The authors have no competing interests to declare that are
relevant to the content of this article.
\end{credits}

%
%
%

\bibliography{main, hatano}
\bibliographystyle{splncs04}
 
\appendix
\numberwithin{equation}{section}

\section{Proofs of Technical Lemmas}\label{sec:appendix}

Throughout this appendix, we assume that $\mathcal{K}\subseteq\Real^{d}$ is a
bounded, closed, convex, and centrally symmetric set, i.e.,
$x\in\mathcal{K}$ implies $-x\in\mathcal{K}$.
Moreover, $\mathcal{K}$ contains the unit ball.
For any $\delta\in(0,1)$, we define the shrunk set
\(
\mathcal{K}_{\delta} \subseteq (1-\delta)\mathcal{K}
= \Bigl\{ x \,\Big|\, \tfrac{1}{1-\delta}x\in\mathcal{K} \Bigr\}.
\)

\subsection{A Property of Dikin Ellipsoids}

Before proving Lemma~\ref{lemm:control-h}, we recall a standard result on Dikin
ellipsoids.

\begin{lemma}[\cite{nemirovski2004interior}]
\label{lemm:dikin}
For any $x\in\operatorname{int}(\mathcal{K})$ and $h\in\Real^{d}$, let
\(
p_x(h)
=
\inf\bigl\{ r\ge 0 \,\big|\, x\pm r^{-1}h \in \mathcal{K} \bigr\}.
\)
One has
\(
p_x(h)
\;\le\;
\|h\|_{x}
\;\le\;
(\nu + 2\sqrt{\nu})\, p_x(h).
\)
\end{lemma}

\subsection{Proof of Lemma~\ref{lemm:control-h}}

\begin{proof}
By Lemma~\ref{lemm:dikin}, for any $x,y\in\mathcal{K}_{\delta}$,
\begin{equation*}
\|y-x\|_{x}
\;\le\;
(\nu + 2\sqrt{\nu})\, p_x(y-x).
\end{equation*}
Hence, it suffices to show that
\[
p_x(y-x)
\le
2\frac{1-\delta}{\delta},
\qquad
\forall\, x,y\in\mathcal{K}_{\delta},
\]
whenever $\delta\le\tfrac{2}{3}$.

First, observe that
\[
x + r^{-1}(y-x) = (1-r^{-1})x + r^{-1}y.
\]
Since $\mathcal{K}$ is convex and $x,y\in\mathcal{K}_{\delta}\subseteq\mathcal{K}$,
we have $x+r^{-1}(y-x)\in\mathcal{K}$ for any $r\ge1$. 
In particular, choosing
    \[
    r = 2\frac{1 - \delta}{\delta} \ge 1,
    \]
    which is equivalent to $\delta \le \tfrac{2}{3}$, ensures above inclusion holds.

Next, consider
\[
x - r^{-1}(y-x) = (1+r^{-1})x - r^{-1}y.
\]
Since $\mathcal{K}_{\delta}=(1-\delta)\mathcal{K}$ and $\mathcal{K}$ is centrally symmetric,
there exist $x',y'\in\mathcal{K}$ such that
\[
x=(1-\delta)x', \qquad y=-(1-\delta)y'.
\]
Substituting yields
\[
(1+r^{-1})x - r^{-1}y
=
(1-\delta)\bigl[(1+r^{-1})x' + r^{-1}y'\bigr].
\]
Choosing $r = 2\frac{1-\delta}{\delta}$ gives
\[
(1-\delta)(1+2r^{-1}) = 1,
\]
and by convexity and symmetry of $\mathcal{K}$,
$(1-\delta)\bigl[(1+r^{-1})x' + r^{-1}y'\bigr]\in\mathcal{K}$.
Thus,
$x - r^{-1}(y-x)\in\mathcal{K}$.

Combining both cases, we conclude that
\[
p_x(y-x)
\le
2\frac{1-\delta}{\delta},
\]
which completes the proof.
\end{proof}

\subsection{Gradient displacement inequality for self-concordant barriers}

Before proving Lemma~\ref{lemm:close-property}, we first recall the
gradient displacement inequality for self-concordant barriers.

\begin{lemma}[Theorem 5.1.8 in \cite{nesterov2018lectures}]
\label{lem:self-concordant-gradient-displacement}
Let \(\mathcal R\) be an $\nu$-self-concordant barrier on
\(\mathcal{K}\in \Real^{d}\). Then, for any
\(x,y\in \mathcal{K}\),
\[
    \left\langle
        \nabla \mathcal R(y)-\nabla \mathcal R(x),\, y-x
    \right\rangle
    \ge
    \frac{\|y-x\|_x^2}
         {1+\|y-x\|_x}.
\]
\end{lemma}

\subsection{Proof of Lemma~\ref{lemm:close-property}}
\begin{proof}
Since $x_t$ is the minimizer of a constrained convex optimization problem,
the first-order optimality condition gives
\[
    \left\langle
        \eta G_{t-1}+\nabla \mathcal R(x_t),
        x-x_t
    \right\rangle
    \ge 0,
    \qquad
    \forall x\in \mathcal K_\delta .
\]
Taking $x=x_{t+1}$, we obtain
\[
    \left\langle
        \eta G_{t-1}+\nabla \mathcal R(x_t),
        x_{t+1}-x_t
    \right\rangle
    \ge 0.
\]
Similarly, the optimality condition for $x_{t+1}$ gives
\[
    \left\langle
        \eta (G_{t-1}+g_t)+\nabla \mathcal R(x_{t+1}),
        x-x_{t+1}
    \right\rangle
    \ge 0,
    \qquad
    \forall x\in \mathcal K_\delta .
\]
Taking $x=x_t$, we have
\[
    \left\langle
        \eta (G_{t-1}+g_t)+\nabla \mathcal R(x_{t+1}),
        x_t-x_{t+1}
    \right\rangle
    \ge 0.
\]

Let
\[
    \Delta_t := x_{t+1}-x_t.
\]
Adding the two inequalities above yields
\[
    \left\langle
        \nabla \mathcal R(x_t)-\nabla \mathcal R(x_{t+1}),
        \Delta_t
    \right\rangle
    -
    \eta\langle g_t,\Delta_t\rangle
    \ge 0.
\]
Equivalently,
\[
    \left\langle
        \nabla \mathcal R(x_{t+1})-\nabla \mathcal R(x_t),
        \Delta_t
    \right\rangle
    \le
    \eta \langle g_t, x_t-x_{t+1}\rangle .
\]
By the Cauchy--Schwarz inequality with respect to the local norm,
\[
    \eta \langle g_t, x_t-x_{t+1}\rangle
    \le
    \eta \|g_t\|_{x_t}^{*}
    \|x_t-x_{t+1}\|_{x_t}.
\]
Thus,
\[
    \left\langle
        \nabla \mathcal R(x_{t+1})-\nabla \mathcal R(x_t),
        x_{t+1}-x_t
    \right\rangle
    \le
    \eta \|g_t\|_{x_t}^{*}
    \|x_{t+1}-x_t\|_{x_t}.
\]

On the other hand, by Lemma \ref{lem:self-concordant-gradient-displacement}, with
\[
    r_t := \|x_{t+1}-x_t\|_{x_t},
\]
we have
\[
    \left\langle
        \nabla \mathcal R(x_{t+1})-\nabla \mathcal R(x_t),
        x_{t+1}-x_t
    \right\rangle
    \ge
    \frac{r_t^2}{1+r_t}.
\]
Combining the two bounds gives
\[
    \frac{r_t^2}{1+r_t}
    \le
    \eta \|g_t\|_{x_t}^{*} r_t.
\]
If $r_t=0$, the claim is immediate. Otherwise, dividing both sides by $r_t$
gives
\[
    \frac{r_t}{1+r_t}
    \le
    \eta \|g_t\|_{x_t}^{*}.
\]
Let
\[
    s_t := \eta \|g_t\|_{x_t}^{*}.
\]
Then
\[
    r_t
    \le
    \frac{s_t}{1-s_t}.
\]
Since $s_t\le 1/2$, we have
\[
    \frac{s_t}{1-s_t}
    \le
    2s_t.
\]
Therefore,
\[
    \|x_{t+1}-x_t\|_{x_t}
    =
    r_t
    \le
    2\eta\|g_t\|_{x_t}^{*}.
\]
This completes the proof.
\end{proof}

\subsection{Proof of Lemma~\ref{lemm:barrier-hessian-inverse-bound}}

\begin{proof}
By the definition of $\mathcal{R}$, we have
$
    \nabla \mathcal{R}(x)
    =
    \frac{2x}{D^2-\|x\|^2}.
$
Therefore, the Hessian is given by
\[
    \nabla^2 \mathcal{R}(x)
    =
    \frac{2}{D^2-\|x\|^2} I
    +
    \frac{4}{\left(D^2-\|x\|^2\right)^2}xx^\top .
\]
Since
$
    xx^\top \succeq 0,
$
we obtain
\[
    \nabla^2 \mathcal{R}(x)
    \succeq
    \frac{2}{D^2-\|x\|^2}I .
\]
Hence,
\[
    \lambda_{\min}\!\left(\nabla^2\mathcal{R}(x)\right)
    \ge
    \frac{2}{D^2-\|x\|^2}.
\]
It follows that
\[
\begin{aligned}
    \lambda_{\max}\!\left(
        \left(\nabla^2\mathcal{R}(x)\right)^{-1}
    \right)
    &=
    \frac{1}{
        \lambda_{\min}\!\left(\nabla^2\mathcal{R}(x)\right)
    }  \\
    &\le
    \frac{D^2-\|x\|^2}{2}.
\end{aligned}
\]
Using
\[
    A_t^2
    =
    b^2\left(\nabla^2\mathcal{R}(x_t)\right)^{-1},
\]
and applying the above bound at $x=x_t$, we obtain
\[
    \|A_t^2\|_2
    \le
    \frac{b^2}{2}
    \left(D^2-\|x_t\|^2\right)
    \le
    \frac{b^2D^2}{2}.
\]
This completes the proof.
\end{proof}

\end{document}